\documentclass[nohyperref]{article}

\usepackage{microtype}
\usepackage{graphicx}
\usepackage{subfigure}
\usepackage{booktabs} 
\usepackage{multirow}
\usepackage{hyperref}
\usepackage{thmtools} 
\usepackage{thm-restate}

\usepackage[preprint]{icml2023_mod}

\usepackage{amsmath}
\usepackage{amssymb}
\usepackage{mathtools}
\usepackage{amsthm}

\def\vphi{\boldsymbol{\phi}}

\def\vb{\mathbf{b}}
\def\vx{\mathbf{x}}
\def\vh{\mathbf{h}}
\def\vp{\mathbf{p}}

\def\vz{\mathbf{z}}
\def\vv{\mathbf{v}}
\def\dd{\mathrm{d}}
\def\di{\mathrm{i}}
\def\rr{\mathbb{R}}
\def\p{\mathrm{p}}
\def\r{\mathrm{r}}

\def\diag{\mathrm{diag}}
\def\czp{{\bf CZP}}

\usepackage[capitalize,noabbrev]{cleveref}

\theoremstyle{plain}

\theoremstyle{definition}

\theoremstyle{remark}

\usepackage[textsize=tiny]{todonotes}

\def\ourtitle{Sample-efficient Surrogate Model for Frequency Response of Linear PDEs using Self-Attentive Complex Polynomials}

\def\soo{scattering coefficients}

\icmltitlerunning{\ourtitle}

\begin{document}

\twocolumn[
\icmltitle{\ourtitle}



\begin{icmlauthorlist}
\icmlauthor{Andrew Cohen$^*$}{comp}
\icmlauthor{Weiping Dou}{comp}
\icmlauthor{Jiang Zhu}{comp}
\icmlauthor{Slawomir Koziel}{uni}
\icmlauthor{Peter Renner}{comp}
\icmlauthor{Jan-Ove Mattsson}{comp}
\icmlauthor{Xiaomeng Yang}{comp}
\icmlauthor{Beidi Chen}{comp}
\icmlauthor{Kevin Stone}{comp}
\icmlauthor{Yuandong Tian$^*$}{comp}

\end{icmlauthorlist}

\icmlaffiliation{comp}{Meta AI}
\icmlaffiliation{uni}{Reykjavic University}

\icmlcorrespondingauthor{Andrew Cohen}{andrewcohen@meta.com}
\icmlcorrespondingauthor{Yuandong Tian}{yuandong@meta.com}

\icmlkeywords{Machine Learning, ICML}

\vskip 0.3in
]



\printAffiliationsAndNotice{\icmlEqualContribution} 

\begin{abstract}
Linear Partial Differential Equations (PDEs) govern the spatial-temporal dynamics of physical systems that are essential to building modern technology. When working with linear PDEs, designing a physical system for a specific outcome is difficult and costly due to slow and expensive explicit simulation of PDEs and the highly nonlinear relationship between a system's configuration and its behavior. In this work, we prove a parametric form that certain physical quantities in the Fourier domain must obey in linear PDEs, named the \emph{\czp{} (Constant-Zeros-Poles) framework}. Applying \czp{} to antenna design, an industrial application using linear PDEs (i.e., Maxwell's equations), we derive a sample-efficient parametric surrogate model that directly predicts its \emph{scattering coefficients} without explicit numerical PDE simulation. Combined with a novel image-based antenna representation and an attention-based neural network architecture, \czp{} outperforms baselines by $10\%$ to $25\%$ in terms of test loss and also is able to find 2D antenna designs verifiable by commercial software with $33\%$ greater success than baselines, when coupled with sequential search techniques like reinforcement learning. 
\end{abstract}

\section{Introduction}
Natural phenomena in mathematical physics such as heat diffusion, wave propagation, electromagnetic radiation, quantum mechanics and many more, are governed by linear Partial Differential Equations (PDEs)~\cite{treves} in the following form:
\begin{equation}
\frac{\partial^n \psi}{\partial t^n} = F(\psi, \nabla_\vx\psi, \ldots;\vh) \label{eq:linear-PDE} 
\end{equation}
where $\psi = \psi(\vx, t)$ is a quantity (e.g., electromagnetic field) that changes over space $\vx$ and time $t$, $F$ is a \emph{linear} function with respect to the quantity $\psi$ and its spatial derivatives of different orders, and $\vh$ is a \emph{design vector} that may nonlinearly determine the linear coefficients of $F$.  

Guided by linear PDEs, designing new physical systems with desired properties is the core practice of modern science and engineering. For example, by finding the shape of reflectors, directors, their relative angles, orientations and electric conductivity, one may design an antenna that can receive signals with specific radio frequency, according to Maxwell's equations. 

Due to the complicated nonlinear dependency between the design vector $\vh$ and the system's final behavior, it often requires large-scale, high fidelity simulation of the PDEs and many years of domain expertise to find an optimal $\vh$. The process is expensive in terms of both simulation and engineer time as engineers often iterate on system configurations using CPU-intensive commercial software \cite {CST,XFDTD}. This high computational overhead is a major bottleneck for rapid experimentation with different structures; A single simulation can take dozens of seconds to several weeks depending on the systematic complexity of a device. 
For this reason, developing a less computationally expensive \emph{surrogate} model to replace explicit simulation is desirable. 

In this work, we take a different path by looking at the temporal Fourier representation of the spatial-temporal quantity $\psi$ under linear PDEs. Surprisingly, it can be proven that its Fourier representation has a specific parametric form: any of its linear combinations, as well as their ratios, can be written as a rational function of complex polynomials with respect to frequency $\omega$, regardless of the specific form of the linear PDEs and initial conditions. 

This key insight, coined as \czp{} (Constant-Zeros-Poles) framework, enables us to develop sample-efficient surrogate models for important industrial-level applications such as antenna design. Specifically, we show that the \emph{scattering coefficients} $S_{11}(\omega)$ of an antenna can be written analytically as the ratio of two complex polynomials of the \emph{same} order, in which their global \textbf{C}onstants, \textbf{Z}eros and \textbf{P}oles are functions of the design choice $\vh$ that can be predicted by neural networks. Inspired by state-of-the-art mesh-based simulation techniques, to properly represent the design choice $\vh$, we propose a novel image-based representation of antennas geometries that captures important boundary information that are traditionally modeled by high resolution meshes. This representation is then tokenized and sent to a transformer-based encoder \cite{Vaswani2017} to capture the non-linear relationship between antenna topology and constants, zeros and poles to be predicted in the \czp{} framework.

Experiments demonstrate a $10\%$ to $25\%$ improvement of the \czp{} framework over multiple baselines on test set loss. Using the parametric form as a domain-specific inductive bias, we achieve better test error with limited data that are expensive to obtain via commercial software. Furthermore, when coupled with an optimization procedure like reinforcement learning, \czp{} can be used to find antenna topologies that meet design specifications, according to commercial EM modeling software, with $33\%$ greater success than baselines with only $40K$ training samples. This shows that \czp{} not only generalizes to unseen designs, but is also less likely to produce overoptimistic regions that the optimization procedure may exploit.

\section{Related Work}
{\bf EM surrogate modeling} In the EM based microwave circuit design, such as microwave filters, impedance-matching networks, multiplexers, etc., the equivalent-circuit, empirical, and semi-analytical models and combinations have been used as surrogate models with the links to the full-wave simulation \cite {Rayas1, Rayas2}. The same approaches have been rarely applied to antenna modeling, due to the fact that the radiating structures are too complex to lend themselves to analytical and/or circuit modeling. Broadly, there are two approaches to antenna surrogate modeling, coarser approximate physics-driven simulation \cite{Zhu_TAP_2007, Slawek_TAP_2013} or data-driven methods which model the computation performed by the simulator \cite {Slawek_L}.

{\bf Deep learning for solving PDEs:} Neural operators are end-to-end methods, formulated to be independent of the underlying mesh discretization and directly approximate the PDE operator between function spaces from samples~\cite{li2020fourier,li2020neural}. These approaches predict the 2D or 3D evolution of systems like Navier-Stokes or Darcy flow but the representations are not used to predict other quantities like the zeros and poles of the $S_{11}$ scattering coefficients. An orthogonal line of work uses supervised~\cite{Pfaff2021,Bhatnagar2019, Guo2016} or sequential methods~\cite{Yang22} for adaptively refining a mesh to model aerodynamics or fluid flow. 


\section{Parametric formula for linear PDEs in the frequency domain}
In this work, we focus on finding the structure of linear PDE in the frequency domain. When solving linear PDE in the form of Eqn.~\ref{eq:linear-PDE}, traditional methods discretize the space and convert the PDEs into the following ODEs~\cite{Weiland1977}:
\begin{equation}
\dot\vphi = A(\vh)\vphi \label{eq:discretized-pde}
\end{equation}
Note that in the original continuous formulation (Eqn.~\ref{eq:linear-PDE}), the quantity $\psi$ is indexed by spatial location $\vx$ and thus is infinite-dimensional. After discretization used in finite difference methods, $\vphi(t)$ is a vector of dimension $N$ at each time $t$, containing the value of $\psi$ (and its spatial derivative) at specific spatial locations. The linear operator $F$ now becomes a matrix $A(\vh)$ of size $N$-by-$N$. Each of its entry is now related to design vector $\vh$ and topological structure of the discretized grid cells. One important property for linear PDEs is that in its discretized form, $A(\vh)$ is a constant and does not change with $\vphi$.

For better understanding, here we put a concrete example of Eqn.~\ref{eq:discretized-pde}. Consider a one-dimensional wave equation $\frac{\partial^2 \psi}{\partial t^2} = c^2 \frac{\partial^2 \psi}{\partial x^2}$. Then by setting $\vphi = [\psi(x_1),\ldots, \psi(x_N), \frac{\partial\psi}{\partial t}(x_1),\ldots, \frac{\partial\psi}{\partial t}(x_N)]^\top \in \rr^{2N}$, the wave equation can be written in the form of Eqn.~\ref{eq:discretized-pde} with $$A = \left[\begin{array}{cc} 
0 & 1 \\
c^2 B & 0
\end{array}\right],$$ where $B\in \rr^{N\times N}$ spatially discretizes the operator $\frac{\partial^2}{\partial x^2}$.

From the initial condition $\phi(\vx,0)$, classic techniques (e.g., finite element methods \cite{Weiland1977}) simply perform temporal integration to get the spatial-temporal signal $\phi(\vx,t)$, from which any quantities that are relevant to the design goals can be computed.  

In this work, we focus on the property of (single-sided) temporal Fourier transform $\hat \vphi(\omega)$ of the spatial-temporal signal $\vphi(t)$:
\begin{equation}
\hat\vphi(\omega) := \int_0^{+\infty} \vphi(t) e^{-\di\omega t} \dd t
\end{equation}
where $\di$ is the imaginary unit and $\omega$ is the frequency. A surprising finding is that, there exists parametric formula for a family of quantities without numerical integration, as presented formally in the following theorem: 
\begin{restatable}{theorem}{formulalinearpde}
\label{thm:formlalinearpde}
For discretized linear PDEs in the form of Eqn.~\ref{eq:discretized-pde}, if $A(\vh)$ is diagonalizable, then any spatially linear combined signals $\vb^\top\hat \vphi(\omega)$ in the Fourier domain is a ratio of two complex polynomials with respect to frequency $\omega$. 
\end{restatable}
This leads to the following corollary that is useful to compute any signal from linear PDEs in the frequency domain:
\begin{restatable}[Parametric Formula for Linear PDEs]{corollary}{czpformulalinearpde}
\label{co:czpformulalinearpde}
For any two linear combination signals $\vb_1^\top\hat \vphi(\omega)$ and $\vb_2^\top \hat \vphi(\omega)$ in Linear PDEs, there exists constant $c_0$, $K_1$ zeros $\{z_k\}$ ($1\le k\le K_1$) and $K_2$ poles $\{p_l\}$ ($1\le l\le K_2$) so that:
\begin{equation}
\frac{\vb_1^\top\hat\vphi(\omega)}{\vb_2^\top\hat\vphi(\omega)} = c_0\prod_{k=1}^{K_1}\left(\omega - z_k\right)\prod_{l=1}^{K_2}\left(\omega - p_l\right)^{-1}
\label{czp_freq_response_linear_pde}
\end{equation}
Note that $c_0$, $\{z_k\}$ and $\{p_l\}$ are all complex functions of the design vector $\vh$ and linear coefficients $\vb_1$ and $\vb_2$.
\end{restatable}
Please check Appendix~\ref{sec:proof} for all proofs. Therefore, parametric forms of many useful quantities can be obtained, e.g., \emph{frequency response} $\phi(\vx_1,\omega)$ given initial condition of linear PDE, \emph{transfer function} $\hat\phi(\vx_1,\omega) / \hat\phi(\vx_2,\omega)$ between two spatial locations $\vx_1$ and $\vx_2$, etc. For any specific quantity (e.g., the scattering coefficients $S_{11}(\omega)$ in antenna design, as mentioned below), learning a neural network that predicts \textbf{c}onstants $c_0$, \textbf{z}eros $\{z_k\}$ and \textbf{p}oles $\{p_l\}$ from the design vector $\vh$ is our proposed \czp{} framework for linear PDEs. 

\section{Application of \czp{} to Antenna Design}\label{prelims}
We now apply our proposed \czp{} framework of linear PDE to Antenna Design problems. Finding antenna design that satisfies the requirement with small dimension, low power consumption and low cost may enable reduction of the physical volume and shape of devices, and leads to seamless wireless connectivity with augmented reality (AR).  



In antenna design, the goal is to find a good spatial configuration of materials in a 2D/3D space, so that the overall system demonstrate a specific property in the frequency domain, e.g., strong absorption at specific frequency. The relationship between an antenna's topology and its properties is governed by the well-known Maxwell's equations which can be written in the form of linear PDEs satisfying Eqn.~\ref{eq:linear-PDE}. In this case, the quantity $\psi$ (and its discretized version $\vphi$) now contains electromagnetic quantities (e.g., electric/magnitude field strength and potentials, voltages and currents, etc) at each discretized grid cell. 

\begin{figure}
    \includegraphics[width=.48\textwidth]{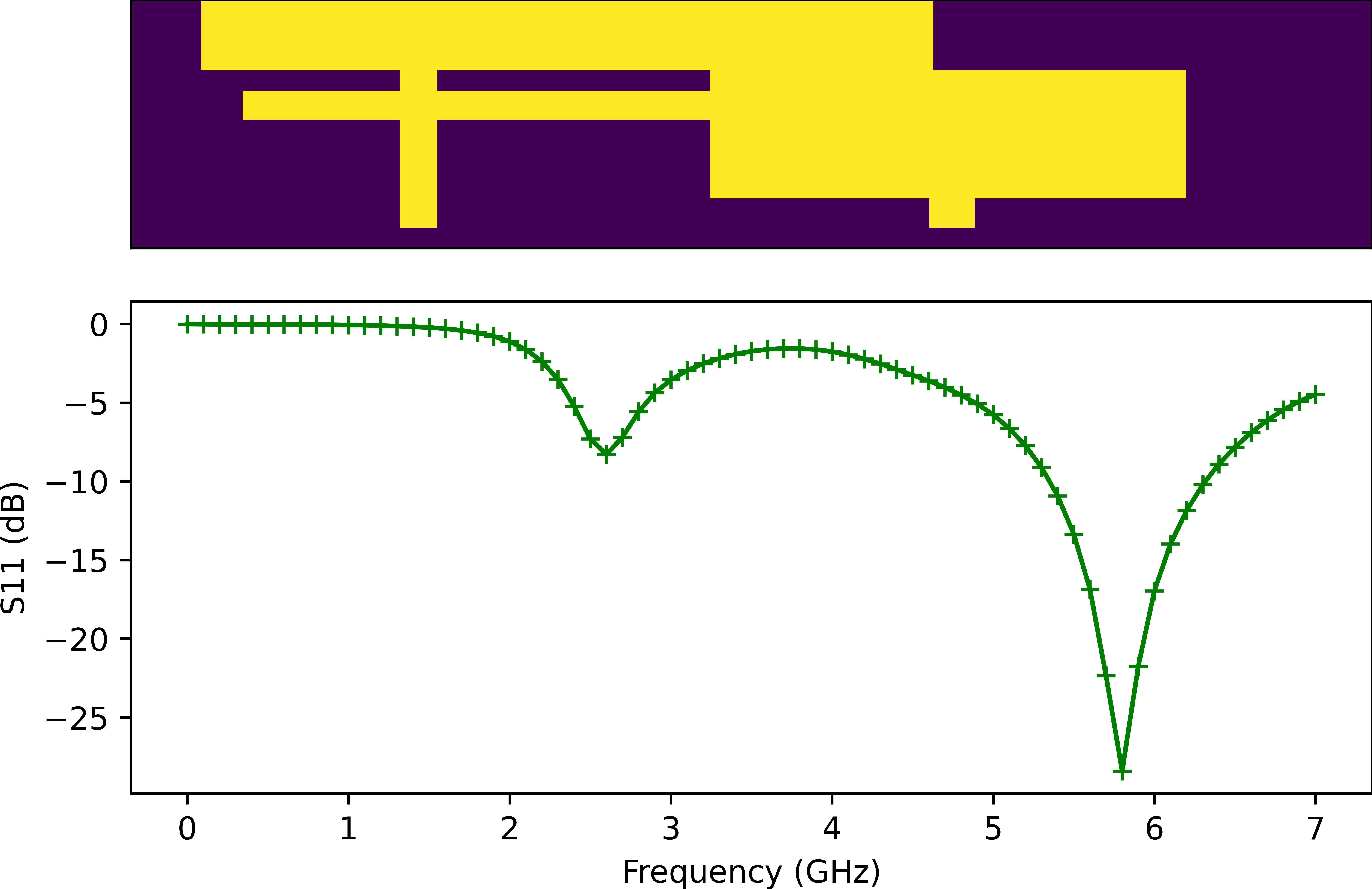}
    \caption{{\bf Top:} An instance of an antenna from the five patch example in a AR device. Yellow corresponds to patches of metallic substrate and purple corresponds to the board on which the antenna sits. {\bf Bottom:} The corresponding frequency response of the given antenna.}
    \label{ant_and_resp}
\end{figure}

\textbf{The Goal of Antenna Design.} Antenna engineers aim to find the right design choice $\vh$ so that specific antenna design targets are met over the frequency bands of interest, for example, there needs to be dips (i.e., more absorption) in $S_{11}(\omega)$ at WiFi 2.4GHz band and WiFi 5-7GHz band for WiFi 6E, as shown in Figure~\ref{ant_and_resp} bottom row.

The frequency properties of the antenna is described by the logarithm of modulus of the \emph{\soo{}} $\log |S_{11}(\omega)|$, typically expressed in decibels (dB). $S_{11}(\omega)$, as a function of frequency $\omega$, is defined as the following~\cite{caspers2012rf}:
\begin{equation}
S_{11}(\omega) := \frac{Z_{\mathrm{in}}(\omega)/Z_0 - 1}{Z_{\mathrm{in}}(\omega)/Z_0 + 1}
\end{equation}
where $Z_{\mathrm{in}}(\omega)$ is the input impedance of the antenna (typically a complex number), determined by the design vector $\vh$. Figure~\ref{ant_and_resp} shows an example antenna design and its $S_{11}(\omega)$. In real-world applications, often other requirements are needed, e.g., low latency, low power consumption and so on, which we leave for future work.  

\subsection{An analytic formula for computing scattering coefficients $S_{11}(\omega)$ of antenna}
Thanks to the insights given by Corollary~\ref{czp_freq_response_linear_pde}, we arrive at an analytic formula for $S_{11}(\omega)$ without performing numerical integration of Maxwell's equations:

\begin{restatable}[Analytical Structure of Scattering Coefficients]{theorem}{czpstruct}
If $A(\vh)$ in discretized Maxwell's equations are diagonalizable, then $\log|S_{11}(\omega)|$ has the parametric form:
\begin{equation}
\log|S_{11}(\omega)| = \log|c_0(\vh)| + \sum_{k=1}^K \log\frac{|\omega - z_k(\vh)|}{|\omega - p_k(\vh)|}
\label{czp_freq_response}
\end{equation}
where the constant $c_0(\vh)$, zeros $\{z_k(\vh)\}_{k=1}^K$ and poles $\{p_k(\vh)\}_{k=1}^K$ are complex functions of the design choice $\vh$.
\end{restatable}
Note that since almost all squared matrices are diagonalizable~\cite{horn2012matrix}, the assumption is not strong. One interesting property is that due to the homogenous structure of $S_{11}(\omega)$, its parametric form has the \emph{same} number of zeros (i.e., roots of the nominator) and poles (i.e., roots of the denominator), reducing one extra hyperparameters to tune. To learn these complex functions, we match the predicted $S_{11}(\omega)$ from the formula with the ground truth one provided by existing commercial software, and train in an end-to-end manner. With this formulation, we avoid any forward numerical integration and arrive at the quantity we want in one inference pass. 

\subsection{Parameterization of design vector $\vh$} 
\label{sec:parameterization-antenna}
In this work, we mainly focus on 2D antenna design (Fig.~\ref{ant_and_resp} top row). We specify our design choice vector $\vh$ as follows:

{\bf Substrate.} The rectangular printed circuitry board of width $S_x$ and height $S_y$ on which the other components sit. The substrate has thickness $S_z$ and dielectric permittivity $\epsilon_r$.

{\bf Ground plane.} A solid rectangle extending through the entire substrate in the $x$ direction and partially in the $y$ direction.

{\bf Discrete port.} The port location is the coordinate $p_x$, $p_y$ and is dependent on one of the front-side metallic patches.

{\bf Front-side metallic patches.} The antenna contains $M$ rectangular metallic patches which can freely move within the substrate area or pre-determined ranges.  The $m$-th patch $p_m$ is defined by its width and height $s_{m,x}$, $s_{m,y}$ and, the coordinate of the bottom left corner $l_{m,x}$, $l_{m,y}$. When the boundary of a metallic patch goes beyond the substrate, the excess is simply clipped. When patches overlap, there is no increase in the thickness; they combine to make a single metallic patch that is no longer rectangular.

Combining all these specifications, we now have an overall design choice vector defined as 
\begin{equation}
\vh = \{S_x,S_y,S_z,p_x,p_y,\{s_{m,x},s_{m,y},l_{m,x},l_{m,y}\}_{m=1}^M\}.
\end{equation}

{\bf Five Patch Example.} In this work, we consider an antenna example with an FR-4 substrate that is $30$mm by $6$mm and 5 front-side metallic patches with fixed dimensions and location boundaries (see Appendix for details). Additionally, we assume that the only degrees of freedom are the locations $\{s_{m,x}, s_{m,y}\}$ of each of the 5 patches as defined by the coordinates of the bottom left corner. We constrain the problem as such for experimental simplicity and acknowledge this is a simplified setting with respect to production-tier antenna optimization. However, the proposed surrogate model is agnostic to the assumption on patches and the optimisation procedure can be easily extended to variable patches of varying dimensions.

\section{Network Architecture}\label{method}
In this section, we discuss the details of the neural network models used to predict the constant, zeros and poles. Specifically, we propose a novel image representation for a 2D antenna inspired by the mesh representations commonly used by EM simulators. Then, following the analysis in the previous section, we introduce our image-based transformer architecture which predicts the zeros and poles of \soo{}.

\subsection{Image representation}
Mesh-based finite element methods underpin many of the available simulation tools in electromagnetics and other fields \cite{Pardo2007}. The mesh converts the underlying PDEs of the system into an ODE solvable by finite element methods \cite{Weiland1977}. Mesh representations use the fact that an antenna's resonance characteristics are directly related to its local and global topological structure. This motivates the use of images for learning a surrogate model as it contains the same local and global spatial information. A model would have to cope with a naive representation (i.e. the coordinates of front-side metallic patches) by learning these spatial relationships.

A critical component of successful meshing is to generate non-uniform, adaptive meshes which allocate high resolution, dense meshing to areas in which the quantity $\vphi$ may change rapidly (e.g., at sharp corners). Adaptive meshing enables the simulation of systems unsolvable by traditional discretization methods \cite{Pfaff2021}. Guided by this, we posit that an image representation of an antenna should provide the key regions (i.e., boundaries and corners of substrate) explicitly so that a neural network does not need to spend unnecessary computation learning these features.

\begin{figure}
    \includegraphics[width=.48\textwidth]{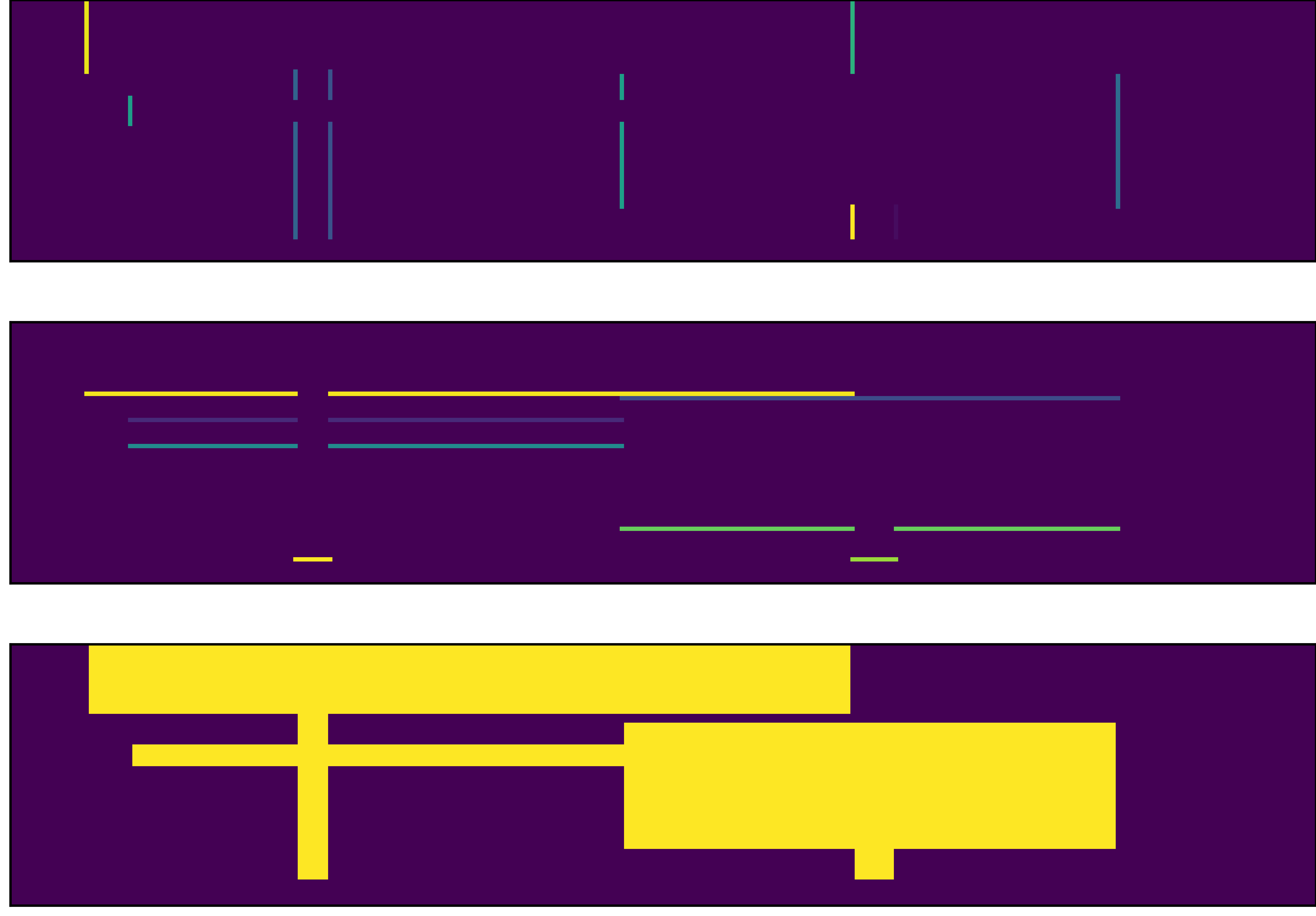}
    \vspace{-.5cm}
    \caption{Three channel image representation. {\bf Top:} Boundary values represent distance to nearest pixel in the $x$-direction.  {\bf Middle:} Boundary values represent distance to nearest pixel in the $y$-direction. {\bf Bottom:} Binary interior of the antenna. This channel does not contain boundaries.}
    \label{fig:image_input}
\end{figure}

We propose a three channel image representation. The first two channels provide the boundary locations in the $x$ and $y$ directions where pixel values $v\in [0,1]$ are floating point to represent the distance to the nearest pixel in the $x$ or $y$ directions, respectively. For example, given the bottom left $(x_{bl}, y_{bl})$ and top right $(x_{tr}, y_{tr})$ floating-number coordinates of a rectangular patch, we compute the pixel indices as the floor, 
\begin{equation*}
    \bar{x}_{bl} = \lfloor x_{bl}\rfloor,\ \bar{y}_{bl} = \lfloor y_{bl}\rfloor,\ \bar{x}_{tr} = \lfloor x_{tr}\rfloor,\ \bar{y}_{tr} = \lfloor y_{tr}\rfloor.
\end{equation*}
Then, we compute the values
\begin{align*}
    v_{l} = 1 - (x_{bl} - \bar{x}_{bl}),\ v_{r} = x_{tr} - \bar{x}_{tr}\\
    v_{b} = 1 - (y_{bl} - \bar{y}_{bl}),\ v_{t} = y_{tr} - \bar{y}_{tr}
\end{align*}

where $v_{l}$, $v_{r}$,$v_{b}$,$v_{t}$ correspond to the left, right, bottom and top boundary values, respectively. The left/bottom boundary is subtracted from $1$ where as the right/top is not because the floor function has a subtly different semantic meaning between these cases; Without the loss of generality, the floor of the left/bottom boundary {\it is not} contained inside the interior of the patch whereas the floor of the right/top {\it is}. We chose this design as it enables sensible image dimensions (i.e. $60 \times 300$ for a $6$mm $\times 30$mm image with a resolution of $10$ pixels to $1$mm) however others are possible. Finally, note, separating the $x$ and $y$ boundaries into two channels enables explicit representation of the corners of patches. 

Finally, a third channel provides the {\it interior} of the antenna as a binary image where $v = 1$ for all index pairs $x,\ y$ such that $x \in [\bar{x}_{bl}+1 , \bar{x}_{tr}-1],\ y \in [\bar{y}_{bl}+1 , \bar{y}_{tr}-1]$ and $v=1$. Please see Algorithm~\ref{imgen} in Appendix~\ref{antenna_specs} for pseudocode of the process and Figure~\ref{fig:image_input} for an example image. Some details are omitted such as patch dimensions which go beyond the board or overlapping patches but these are straightforwardly handled via clipping and masking.


\subsection{Surrogate Model}
In this section, we propose an architecture for a surrogate model which predicts the zeros and poles directly from the image representation which is then used to compute \soo{}. The architecture is based on the Visual Transformer~\cite{Wu20} which is motivated by the insight that local, spatial components such as boundaries between substrate of the antenna should be tokenized and then used by a transformer~\cite{Vaswani2017} to compute the global characteristics.

Given the input image $\mathbf{I} = \mathbf{I}(\vh) \in \rr^{3 \times HW}$ as a function of design choice $\vh$, we first augment it with two additional channels of linearly spaced $x$ and $y$ coordinates~\cite{Liu2018}, to yield augmented image $\hat{\mathbf{I}} \in \rr^{5\times HW}$. This is because the specific location of antenna components, in addition to its topology, determines the corresponding frequency response. 
Then, a CNN takes $\hat{\mathbf{I}}$ as an input and generates feature maps $\mathbf{X} \in \rr^{HW\times C}$ where $H$, $W$ and $C$ are the height, width and channel dimension, respectively. A filter-based tokenizer \cite{Zhang2019} generates $L$ visual tokens $\mathbf{T} \in \rr^{L\times C}$ by mapping each pixel via a pointwise convolution to $L$ groups with matrix $\mathbf{W} \in \rr^{C\times L}$ and computes a softmax in the pixel dimension
\begin{equation*}
    \mathbf{A} = \mathrm{Softmax}_{HW}(\mathbf{X}\mathbf{W})
\end{equation*}
where $\mathbf{A} \in \mathbb{R}^{HW\times L}$ is referred to as an attention map. Visual tokens $\mathbf{T}$ are computed via $\mathbf{T} = \mathbf{A}^T\mathbf{X}$ which is the weighted average of pixels in the original feature map $\mathbf{X}$. 
Intuitively, the tokens $\mathbf{T}$ capture semantics such as relative boundary and corner locations and from this the transformer computes the global characteristic of the antenna configuration.  Please see Figure~\ref{attn_maps} in Appendix~\ref{qualitative} for a subset of the learned attention maps for a specific antenna instance which demonstrate this.

After that, $\mathbf{T}$ is then passed through a multi-layer transformer encoder \cite{Vaswani2017}, flattened and passed through a fully connected layer and a non-linearity. From this representation, three separate complex-valued fully connected layers predict the constant, zeros and poles. Concretely, let $C_\theta,\  Z_\theta,\ P_\theta$ be linear layers parameterized by $\theta$. Then,
\begin{align*}
    \vv = \mathrm{FC}(\mathrm{Transformer}(\mathbf{T}))\\
    c_0(\vh) := C_\theta(\vv),\ \mathbf{z}(\vh) := Z_\theta(\vv),\ \mathbf{p}(\vh) := P_\theta(\vv)
\end{align*}
where $\vh$ is the design choice and $c_0(\vh),\ \mathbf{z}(\vh),\ \mathbf{p}(\vh)$ are the constant and vectors (of length $K$) of zeros and poles, respectively, used to compute the frequency response as per Equation~\ref{czp_freq_response}. We refer to this architecture which outputs the constant, zeros and poles as \czp{} models / architectures, or just \czp{} as abbreviation.

\subsection{Model training}
When training \czp{} models, we do not have direct supervision to $c_0(\vh)$, zeros $\vz(\vh)$, and poles $\vp(\vh)$, but only the $S_{11}(\omega)$ provided with CST Microwave Studio~\cite{CST} as ground truth. Therefore, we leverage Eqn.~\ref{czp_freq_response} to compute estimated $S_{11}(\omega)$ with $c_0$, $\vz$, and $\vp$, so that it can match the ground truth. We then train the model via back-propagation in an end-to-end manner to minimize the Mean Squared Error (MSE), for frequencies in the range $[0.2-7.0]$GHz at increments of $0.1$ (i.e., 69 dimensions). We use a shrinkage loss~\cite{Lu2018} variant of MSE as we found that with vanilla MSE, the model had higher error on the crucial parts of the \soo{} (i.e., the resonances). 

\section{Experiments}
In this section, we demonstrate the impact of our architectural choices and image representation on the five patch example antenna discussed in Section~\ref{sec:parameterization-antenna}. Specifically, we show that:
\begin{itemize}
    \item Our proposed image representation is a significant improvement over reasonable coordinate-based inputs as well as a naive binary image input.
    \item The \czp{} formulation outperforms raw prediction when using the same transformer architecture proposed in Section~\ref{method} and the proposed image representation.
    \item The transformer architecture outperforms a CNN for the image representation and an MLP for coordinate based input.
    \item \czp{} generalizes well to unseen antenna designs, not only on a held-out dataset, but also as a surrogate model for designs proposed by the reinforcement learning (RL) based search procedure, as verified to meet specific resonance requirements by commercial softwares.  
\end{itemize}

\newcommand{\seq}{\,{=}\,}

\subsection{Surrogate Modeling}
We use $48K$ total samples, uniformly sampled and simulated with CST Microwave Studio \cite{CST} where each sample takes between $90$ and $120$ seconds to simulate. $90\%$ of samples are used for training and $10\%$ are used for testing. From the training set, $10\%$ are randomly sampled and used for validation. Each experiment is run for 3 random seeds. Appendix~\ref{hyperparams} provides all experimental hyperparameters.

Figure~\ref{raw_vs_czp} illustrates the first set of experiments in which we demonstrate the effectiveness of (1) our novel image representation and (2) the proposed \czp{} when using the transformer architecture. To show (1), we compare against a coordinate-based method which concatenates the normalized bottom-left $x,y$ coordinate of each patch with a one-hot vector to distinguish between patches. When using the transformer architecture, this generates $5$ tokens, each with dimension $7$. Before being processed by the transformer, each token is projected into a $256$ dimensional vector by a $2$-layer MLP with hidden layer of width $256$.  Additionally, to demonstrate (2), for both coordinate and image input, we compare against directly predicting the raw $69$ dimensional frequency response with a fully connected layer, referred to as Raw in figures. Additionally, we ablate over different degree $K$ of \czp{} with values $8$, $12$, $16$, $20$ and the number of attention layers $L$ with values $8$, $6$, $4$, $2$.

First, within each configuration, images improve over its coordinate counterpart by a minimum of $9.6\%$ with $L\seq4$ and raw prediction and a maximum of $26.8\%$ with $L\seq2$ and $K\seq12$. Second, for the image representation, \czp{} improves over raw prediction by a minimum of $9.2\%$ with $L\seq8$ and $K\seq12$ and a maximum of $28.0\%$ with $L\seq2$ and $K\seq12$. For the coordinate representation, \czp{} improves over raw prediction by a minimum of $4.6\%$ with $L\seq8$ and $K=12$ and a maximum of $25.1\%$ with $L=2$ and $K\seq8$. Finally, increasing the transformer depth from 2 to 8 layers improves raw prediction for image and coordinate representations by $29.5\%$ and $35.9\%$, respectively. Increased depth improves the \czp{} an average of $15.4\%$ and $20.9\%$ for image and coordinate representations, respectively.    

From these statistics, we can extract the following insights which support the \czp{} architecture and image representation as powerful inductive biases; (1) With fewer transformer layers, \czp{} yields greater improvement over raw prediction and (2) \czp{} and image representation benefit {\it the least} from increasing the complexity of the model and, at the opposite extreme, raw prediction and coordinate representation benefit the most. These two points show that without these inductive biases, deeper models are required as shallower models are likely to fall into local minima.

\begin{figure*}[ht]
\centering
  \begin{tabular}{| c | c | c | c | c | c | c |}
    \hline
        Layers&Input Type&Raw & \czp{} $K\seq8$ & \czp{} $K\seq12$ & \czp{} $K\seq16$ & \czp{} $K\seq20$\\\hline
        
        \multirow{2}{*}{$L\seq8$}& Image &$.00284 \pm 7\mathrm{e}{-5}$&$.00243 \pm 1\mathrm{e}{-4}$&$.00258 \pm 3\mathrm{e}{-5}$&$.00253 \pm 5\mathrm{e}{-5}$& $\mathbf{.00234 \pm 2\mathrm{e}{-5}}$\\ \cline{2-7}
        
        &Coord & $.00327 \pm 5\mathrm{e}{-5}$ & $.003\pm 8\mathrm{e}{-5}$& $.00312\pm 4\mathrm{e}{-5}$& $.00307 \pm 4\mathrm{e}{-5}$ & $.00303 \pm 5\mathrm{e}{-5}$ \\
       \hline
        \multirow{2}{*}{$L\seq6$}& Image &$.00312 \pm 9\mathrm{e}{-5}$&$.00254 \pm 3\mathrm{e}{-5}$&$.00253 \pm 9\mathrm{e}{-5}$&$.00255 \pm 6\mathrm{e}{-5}$& $\mathbf{.00249 \pm 8\mathrm{e}{-5}}$\\ \cline{2-7}
        
        &Coord & $.00357 \pm 8\mathrm{e}{-5}$ & $.00308\pm 9\mathrm{e}{-5}$& $.00309\pm7\mathrm{e}{-5}$& $.00313 \pm 8\mathrm{e}{-5}$ & $.00303 \pm 7\mathrm{e}{-5}$ \\
       \hline
       \multirow{2}{*}{$L\seq4$}& Image & $.00348 \pm 1\mathrm{e}{-4}$&$.00268 \pm 7\mathrm{e}{-5}$&$.00266 \pm 2\mathrm{e}{-4}$&$\mathbf{.00251 \pm 5\mathrm{e}{-5}}$&$.00252 \pm 1\mathrm{e}{-4}$ \\ \cline{2-7}
        
        &Coord & $.00385 \pm 5\mathrm{e}{-5}$ & $.00317\pm6\mathrm{e}{-5}$& $.00328\pm8\mathrm{e}{-5}$& $.0032 \pm 1\mathrm{e}{-4}$ & $.00322 \pm 2\mathrm{e}{-5}$ \\
       \hline
       \multirow{2}{*}{$L\seq2$}& Image &$.00403 \pm 4\mathrm{e}{-5}$&$.00292 \pm 1\mathrm{e}{-4}$&$\mathbf{.0029 \pm 1\mathrm{e}{-4}}$&$.00294 \pm 1\mathrm{e}{-4}$& $.00292 \pm 2\mathrm{e}{-4}$\\ \cline{2-7}
        
        &Coord & $.0051 \pm 8\mathrm{e}{-5}$ & $.00382\pm2\mathrm{e}{-4}$& $.00396\pm1\mathrm{e}{-4}$& $.00385 \pm 5\mathrm{e}{-5}$ & $.00384 \pm 9\mathrm{e}{-5}$ \\
       \hline
  \end{tabular}
\caption{Mean and standard deviation of the test loss over 3 seeds with the transformer architecture for {\it image} and {\it coordinate} input representations and $L\seq8,6,4,2$ attention layers. Results are reported for raw frequency prediction and the \czp{} architecture with degree $K\seq8,12,16,20$. In all configurations, the image representation outperforms coordinates and \czp{} outperforms raw prediction.}\label{raw_vs_czp}
\end{figure*}

In Figure~\ref{naive_baselines}, we provide results for $4$ other baselines to show the impact of the transformer and image representation over reasonable alternatives e.g., a fully connected MLP with coordinate input, a CNN with our image input, the transformer with a naive single-channel binary image input, and the Fourier Neural Operator (FNO)~\cite{li2020fourier} developed to solve other PDEs such as Navier-Stokes with image input. In the last row we reproduce the results of the 8-layer transformer with image input from Figure~\ref{raw_vs_czp}.  The 8 layer transformer is a $40\%+$ improvement on these baselines.

\begin{figure*}[ht]
\centering
  \begin{tabular}{| c | c | c | c | c | c |}
    \hline
    Arch + Input & Raw & \czp{} $K\seq8$ & \czp{} $K\seq12$ & \czp{} $K\seq16$ & \czp{} $K\seq20$\\ \hline
    MLP + Coord & $.00492 \pm 5\mathrm{e}{-5}$& $.00502 \pm 1\mathrm{e}{-4}$& $.00553 \pm 3\mathrm{e}{-4}$&$0.00507 \pm 3\mathrm{e}{-4}$& failed \\
    \hline
    CNN + Image &$.0054 \pm 3\mathrm{e}{-5}$&$.00496 \pm 1\mathrm{e}{-3}$&$.00405 \pm 1\mathrm{e}{-4}$& $.00424 \pm 1\mathrm{e}{-4}$& failed\\
    \hline
    Transformer +&\multirow{2}{*}{$.0049 \pm 1\mathrm{e}{-4}$}&\multirow{2}{*}{$.005 \pm 2\mathrm{e}{-4}$}& \multirow{2}{*}{$.00488 \pm 9\mathrm{e}{-5}$}&\multirow{2}{*}{$.00501 \pm 1\mathrm{e}{-4}$} & \multirow{2}{*}{$.0049 \pm 8\mathrm{e}{-5}$}\\ 
    Binary Image &&&&&\\
    \hline
     FNO + &\multirow{2}{*}{$.00724 \pm 5\mathrm{e}{-5}$}& \multirow{2}{*}{$.00715 \pm 1\mathrm{e}{-4}$}&\multirow{2}{*}{$.00706 \pm 5\mathrm{e}{-5}$} & \multirow{2}{*}{$0.0073 \pm 2\mathrm{-4}$} & \multirow{2}{*}{$.00724 \pm 9\mathrm{-5}$}\\ 
    Image &&&&&\\
    \hline
     {\bf Transformer +}&\multirow{2}{*}{$.00284 \pm 7\mathrm{e}{-5}$}&\multirow{2}{*}{$.00243 \pm 1\mathrm{e}{-4}$}&\multirow{2}{*}{$.00258 \pm 3\mathrm{e}{-5}$}&\multirow{2}{*}{$.00253 \pm 5\mathrm{e}{-5}$}& \multirow{2}{*}{$.00234 \pm 2\mathrm{e}{-5}$}\\
     {\bf Image} &&&&&\\
    \hline
     
  \end{tabular}
\caption{Mean and standard deviation of the test loss over 3 seeds for ablations of architectural components of the proposed model and baselines. Results reported are for raw prediction and \czp{} with degree $K\seq8,12,16,20$ for the following configurations: 6-layer MLP with coordinate input, 5 layer CNN with image input, and 8 layer transformer with binary image input and 4-layer FNO with image input.}\label{naive_baselines}
\vspace{-.25cm}
\end{figure*}

\begin{figure}
\centering
  \begin{tabular}{| c | c | c |}
    \hline
    $\%$ Training &\multicolumn{2}{c |}{Transformer Out} \\ 
    Data &  Raw & \czp{} $K\seq20$ \\ \hline
    $25\%$ &$.00621 \pm 3\mathrm{e}{-4}$& $.00574 \pm 1\mathrm{e}{-4}$\\ \hline
    $50\%$ &$.00387 \pm 2\mathrm{e}{-4}$& $.0038 \pm 2\mathrm{e}{-4}$\\ \hline
    $75\%$ &$.00329 \pm 1\mathrm{e}{-4}$& $.00296 \pm 2\mathrm{e}{-4}$\\ \hline
    $100\%$ &$.00284 \pm 7\mathrm{e}{-5}$& $.00234 \pm 2\mathrm{e}{-5}$\\ \hline
  \end{tabular}
\caption{Mean and standard deviation of the test loss over 3 seeds for \czp{} $K\seq20$ and raw prediction with the 8-layer transformer architecture and image input with randomly sampled subsets of the training data for portions $25\%$, $50\%$ and $75\%$. \czp{} has a lower test loss.}\label{data_ablation_mse}
\vspace{-.5cm}
\end{figure}
Finally, in Figure~\ref{data_ablation_mse}, we provide a data ablation with raw prediction and \czp{} $K\seq20$. The trend of \czp{} out performing raw prediction holds in this setting as well although the differences in test loss are small. However, in the next section, we show that our model greatly outperforms the baselines when used for optimization when trained with less data, more robust for unseen designs. For other qualitative results such as attention map visualizations, please see Appendix~\ref{qualitative}.

\subsection{Optimization}
In this section, we demonstrate the utility of the proposed model by showing it can be used by an optimization procedure to find antenna configurations that have specific resonance characteristics. This is a significant test of the generalization and robustness of the model since (1) an antenna with the desired resonances is {\it not} contained in the training set and (2) an optimization procedure can very easily find adversarial configurations to exploit the weaknesses of the surrogate model~\cite{Yuan2017}. We hypothesize that \czp{} will be far more robust than raw prediction to these kinds of samples because it is by design smooth (i.e., a ratio of two polynomials) whereas raw prediction has no built in bias encouraging this property. Please see Figure~\ref{smooth_vs_nonsmooth} in Appendix~\ref{qualitative} for qualitative intuition regarding this. In this section, we provide results which demonstrate that our proposed model, when used by an optimization procedure, has a significantly higher success rate and is more robust to dataset size than the baseline.

\begin{figure*}
\begin{tabular}{c c}
\includegraphics[width=.5\textwidth]{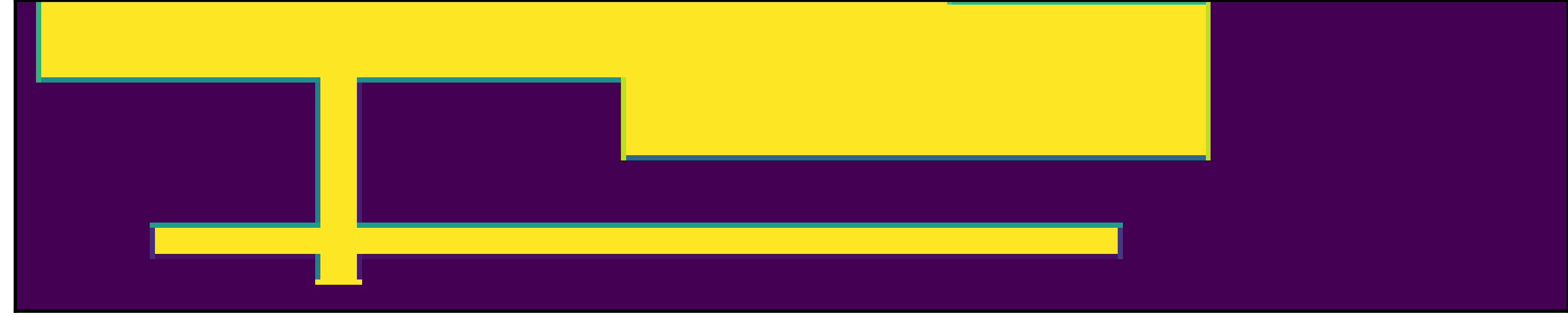} & \includegraphics[width=.5\textwidth]{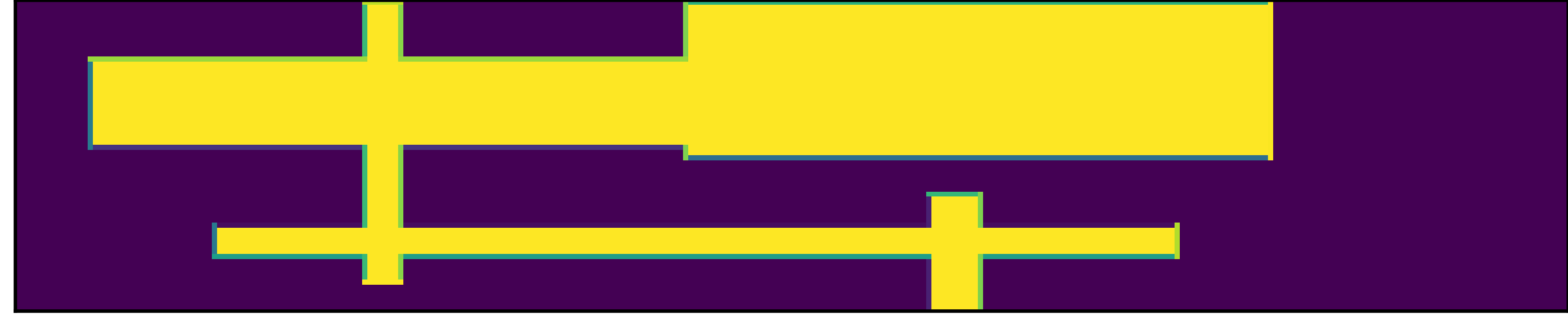} \\
\includegraphics[width=.5\textwidth]{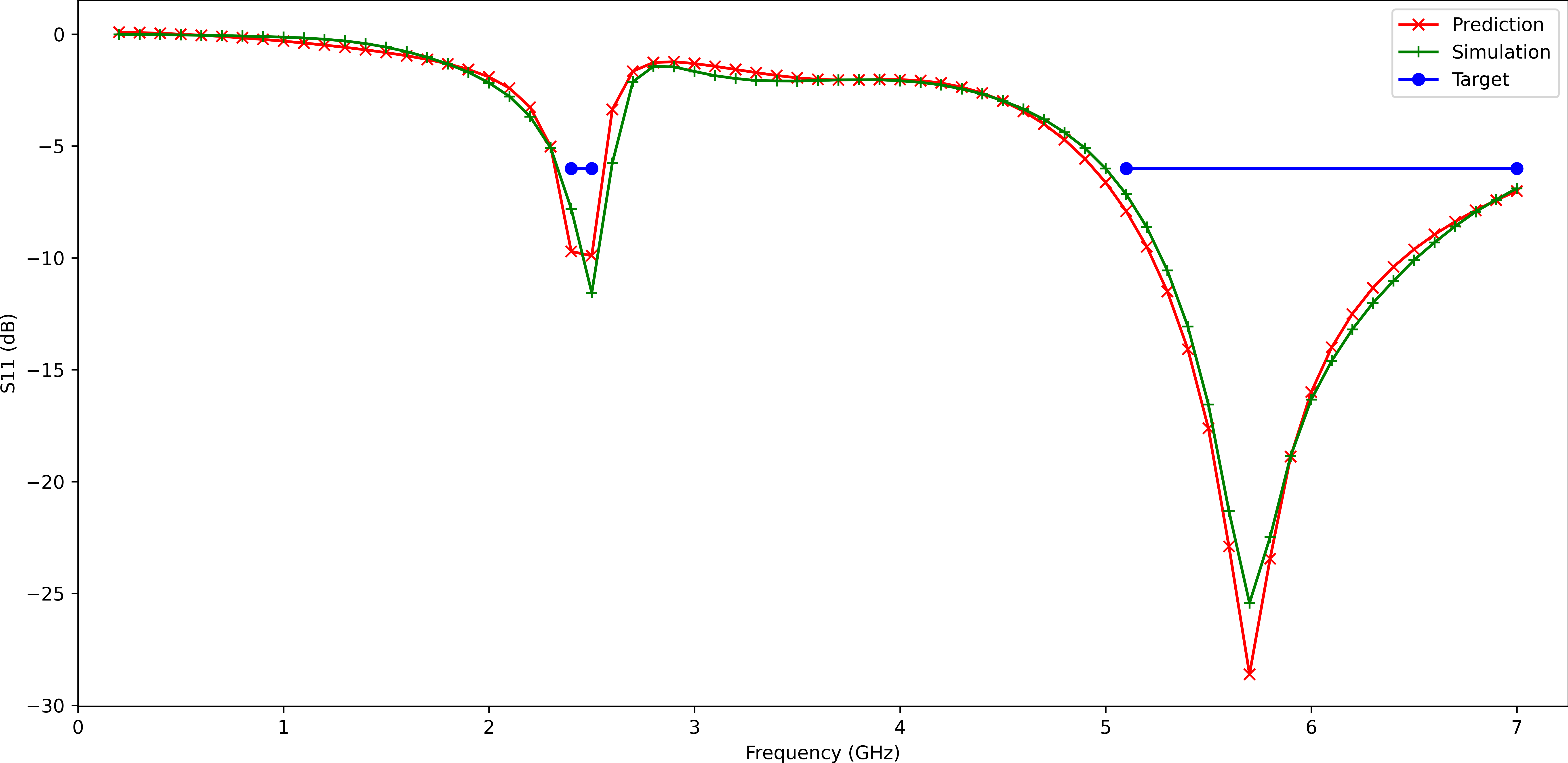} & \includegraphics[width=.5\textwidth]{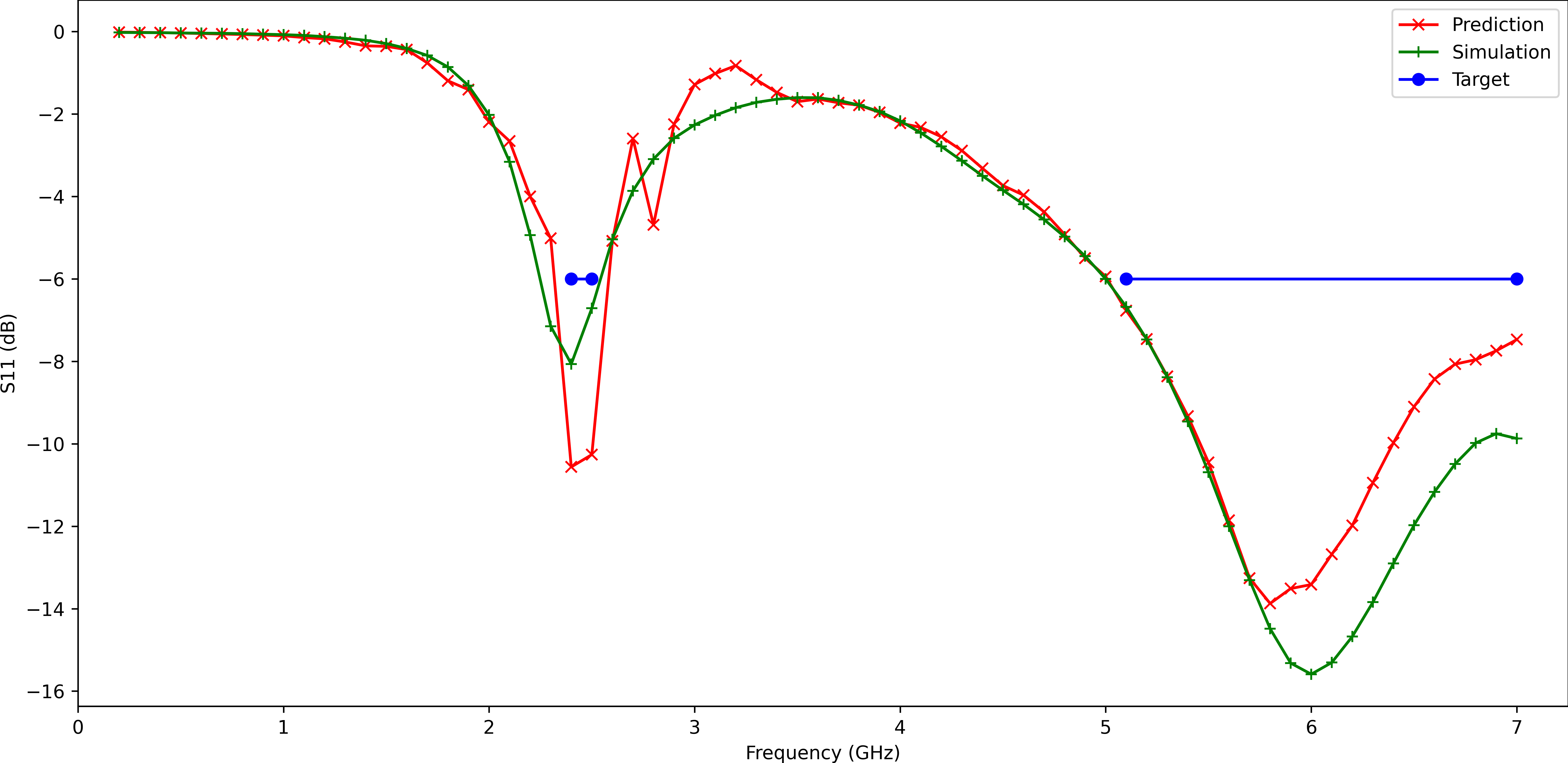} \\
{\bf (a)} \czp & {\bf (b)} Raw \\
\end{tabular}
\vspace{-.25cm}
\caption{Two successful antenna configurations (top row) and corresponding frequency responses (bottom row) predicted by the model (red) and computed by CST (green) found by optimization via RL for {\bf (a)} \czp{} and {\bf (b)} Raw.}\label{opt_examples}
\vspace{-.25cm}
\end{figure*}
\begin{figure}
\centering
  \begin{tabular}{| c | c | c | c| c|}
    \hline
    \multirow{3}{*}{$\%$ Data} &\multicolumn{4}{c |}{Transformer Out} \\ 
    &  \multicolumn{2}{c|}{Raw} & \multicolumn{2}{c|}{\czp{} $K\seq20$} \\ 
    &  $\%$ Top 3 & Any Top 3 & $\%$ Top 3 & Any Top 3 \\ \hline
    $25\%$ &$11.1\%$ & $22.2\%$&$33.3\%$&$55.6\%$\\ \hline
    $50\%$ &$14.8\%$&$22.2\%$&$40.7\%$&$66.7\%$\\ \hline
    $75\%$ &$33.3\%$&$55.6\%$&$55.6\%$&$77.8\%$\\ \hline
    $100\%$ &$51.9\%$&$66.7\%$&$88.9\%$&$100\%$\\ \hline
  \end{tabular}
\caption{Success rate for the $\%$ of top 3 configurations and if any of the top 3 configurations meet design requirements. Experiments performed for 3 seeds for RL for each of the 3 seeds of \czp{} $K\seq20$ and raw prediction with the 8-layer transformer architecture and image input from the previous section. Experiments conducted also with randomly sampled subsets of the training data for portions $25\%$, $50\%$ and $75\%$. \czp{} is more robust than raw prediction to the optimization procedure and to dataset size.}\label{data_ablation_rate}
\vspace{-.5cm}
\end{figure}
We frame antenna design as a reinforcement learning (RL)~\cite{sutton} problem where an agent is tasked with sequentially placing each of the 5 patches such that the frequency response of the final antenna meets the resonance characteristics. Recall from Section~\ref{prelims}, this means that the corresponding $S_{11}$ is below a certain threshold at specific frequency ranges. In this problem, the frequency ranges are $2.4$ GHz-$2.5$ GHz and $5.1$ GHz-$7.0$GHz and the target thresholds are $t_{[2.4-2.5]} = -6.0$ dB and $t_{[5.1-7.0]} = -6.0$ dB, the spectrum for WiFi 6E.

Formally, we define the state and action of the Markov Decision Process (MDP)~\cite{puterman1994markov} as: 
\begin{itemize}
\item {\bf State:} A one-hot identifier and $(x, y)$ coordinates of the bottom left corner of the patches which have been placed and a one-hot vector for the next patch to be placed.
\item {\bf Action:} $(x, y)$ coordinates of the bottom-left corner of the next patch to be placed.
\end{itemize}

After all patches have been placed, the coordinates are converted to the image representation and the surrogate model predicts the frequency response $\log|S_{11}(\omega)|$. From the final $\log|S_{11}(\omega)|$, we compute the following reward components for each resonance target.
\begin{align}
r_{[2.4-2.5]} = \min(t_{[2.4-2.5]} - \log|S_{11}(\omega)|_{[2.4-2.5]})\\
r_{[5.1-7.0]} = \min(t_{[5.1-7.0]} - \log|S_{11}(\omega)|_{[5.1-7.0]})
\end{align}
where the subscripts correspond to list slicing. The sum $r = r_{[2.4-2.5]} + \min(1.0, r_{[5.1-7.0]})$ is then the reward given at the final timestep and at all previous timesteps the reward is zero. Note, we prevent the second reward component from being greater than $1.0$ because in experiments the higher band ($5.1$ GHz-$7.0$ GHz) seemed to be easier to optimize and often led to local minima that did not optimize the lower band ($2.4$ GHz-$2.5$ GHz). To optimize, we use the implementation of Soft Actor Critic (SAC)~\cite{SAC} from Stable-Baselines3~\cite{stable-baselines3} and build the environment using the Gym API~\cite{gym}. Default hyperparameters are used except we perform two updates at the end of each episode as opposed to one or more updates per step. 

For these experiments, we use the \czp{} $K\seq20$ and raw prediction architectures with an 8-layer transformer as these achieved the lowest test losses. For each of the 3 seeds for each architecture trained in the previous section, we run 3 seeds of optimization for a total of 9 experiments per configuration. In each experiment, we deploy the SAC agent for $25$K total episodes or $125$K total timesteps (since the agent places 1 of 5 patches each step). We also investigate the robustness of this process to dataset size which is critical in antenna design as sample collection is expensive.

Generally, SAC is able to find configurations which meet the requirements with both architectures and in Figure~\ref{opt_examples} we provide examples. The top row provides the found antenna configuration and the bottom row the frequency responses predicted by the model (red) and CST (green).

However, in terms of {\it success rate} (i.e., how many configurations or optimization runs actually produce an antenna which meets the constraints), \czp{} significantly out performs raw prediction. Specifically, in Figure~\ref{data_ablation_rate}, we provide the percentage of the top 3 configurations (i.e., 3 per seed for a total of 27) found over all seeds which meet the constraints and also the percentage of runs where {\it any} of the top 3 meet the constraints. Additionally, we perform a data ablation to show that our \czp{} model is more robust to less data demonstrating its strength as an inductive bias.

\section{Conclusion}
In this work, we theoretically derived a novel parametric form physical quantities must obey in linear PDEs in terms of complex-valued zeros and poles. Based on this, we proposed the the \czp{} framework, which uses a neural network to predict these zeros and poles. Applying this to the problem of industrial antenna design, we show that an antenna's frequency response can be predicted by our \czp{} model and propose an efficient novel image representation of an antenna from which to do so. 
We then demonstrated experimentally that \czp{} and the image representation are significant advances through architecture and data ablation studies. Finally, we showed that our \czp{} model has significantly higher utility in terms of success rate for optimization of antenna design than baselines.

Although the results are significant, the problem investigated in this work is still relatively simple compared to production level antenna systems. Future work will involve solving more complicated 2D problems as well as generalizing to 3D antenna. Additionally, in this line of work, we plan to explore other tokenization schemes that are as information rich as images but are more computationally efficient since images require convolutions to featurize. Lastly, future work will involve more applications involving  frequency responses via solving linear PDEs.

\bibliography{antenna_ai}
\bibliographystyle{icml2023}
\clearpage

\appendix

\section{Derivations}
\label{sec:proof}
\formulalinearpde*
\begin{proof}
Consider the following high-dimensional linear ODE problem of $N$ variables:
\begin{equation}
\dot\vphi = A\vphi \label{eq:dynamics-phi}
\end{equation}
Here $\vphi$ is a $N$-dimensional vector and we refer its component at spatial location $\vx$ as $\phi(\vx)$ (i.e. a scalar), and $A$ is an $N$-by-$N$ diagonalizable matrix. $A = A(\vh)$ depends on material and topological properties, i.e., the design choice $\vh$, is time-invariant and is not necessarily symmetric. Therefore, $A$ has the following decomposition $A = U\Lambda U^{-1}$, where each column of $U$ is its eigenvector and $\Lambda = \diag(\lambda_1, \ldots, \lambda_N)$ is a diagonal matrix containing all its  eigenvalues $\{\lambda_i\}_{i=1}^N$. Note that entries of both $U$ and $\Lambda$ can be complex numbers. Since the ODE is stable and all excitation eventually vanishes, the real part of all eigenvalues are negative: $\Re[\lambda_i] < 0$. 

By theory of linear ODE, we know that the solution to Eqn.~\ref{eq:dynamics-phi} has analytic form:
\begin{equation}
\vphi(t) = e^{At} \vphi(0)
\end{equation}
where $\vphi(0)$ is the initial condition of $\vphi$. Then we can compute its (single-sided) Fourier transform $\hat\vphi(\omega) := \int_0^{+\infty} \vphi(t) e^{-\di\omega t} \dd t$:
\begin{eqnarray}
    \hat\vphi(\omega) &=& \left(\int_0^{+\infty} e^{At} e^{-\di\omega t} \dd t\right) \vphi(0) \nonumber \\
    &=& U \left(\int_0^{+\infty} e^{\Lambda t} e^{-\di\omega t} \dd t\right) U^{-1} \vphi(0) \nonumber \\
    &=& U \diag\left(\int_0^{+\infty} e^{\lambda_i t} e^{-\di\omega t} \dd t\right) U^{-1} \vphi(0)  \nonumber \\
    &=& U \diag\left(\frac{1}{\di\omega-\lambda_i}\right) U^{-1} \vphi(0) 
\end{eqnarray}
Note that $U$ and $\vphi(0)$ are all time-invariant. 

For double-sided Fourier transform $\hat\vphi(\omega) := \int_{-\infty}^{+\infty} \vphi(t) e^{-\di\omega t} \dd t$ and symmetric signal extension $\vphi(-t) = \vphi(t) = e^{A|t|}\vphi(0)$ to negative side, similar we can compute 
\begin{eqnarray}
    \int_{-\infty}^{+\infty} e^{\lambda_i |t|} e^{-\di\omega t} \dd t &=& \frac{1}{\di\omega-\lambda_i} + \frac{1}{-\di\omega-\lambda_i} \\
    &=& -\frac{2\lambda_i}{\lambda^2_i + \omega^2}
\end{eqnarray}
Note that while it looks like a real number, the result is still complex since the eigenvalue $\lambda_i$ is in general a complex number for an asymmetric dynamical system $A$. In this case, we can write $\hat\vphi(\omega)$ as:
\begin{equation*}
    \hat\vphi(\omega) = -U \diag\left[\frac{2\lambda_1}{\lambda_1^2 + \omega^2}, \ldots, \frac{2\lambda_N}{\lambda_N^2 + \omega^2}\right] U^{-1} \vphi(0)
\end{equation*}

In both cases, each component of $\hat\vphi(\omega)$ is a rational function of complex polynomial with respect to frequency $\omega$ and so does any of its linear combinations $\vb^\top\hat\vphi(\omega)$. 
\end{proof}

\czpformulalinearpde*
\begin{proof}
According to Theorem~\ref{thm:formlalinearpde}, both the linear combinations $\vb_1^\top\hat\vphi(\omega)$ and $\vb_2^\top\hat\vphi(\omega)$ are in the form of rational function of complex polynomials, therefore, their ratio is also a rational function of complex polynomials. According to the fundamental theorem of algebra, any polynomial of order $K$ can be written as a product of order-1 factors $\omega-\omega_k$ and a constant, where $\{\omega_k\}$ are the (complex) roots of the $K$-th order polynomials. Applying it to both the nominator and the denominaotr of the rational function and we reach the conclusion.
\end{proof}

\czpstruct*
\begin{proof}
Since all voltages and currents in the antenna are linear functions of EM quantities represented in the components of $\vphi$, by Theorem~\ref{thm:formlalinearpde}, they are also rational function of complex polynomials. As a result, the input impedance $Z_{\mathrm{in}}(\omega)$ of the antenna, defined as the ratio between the voltage and the current, is also a rational function of complex polynomials, represented as a quotient of two complex polynomials $Q_1(\omega)$ and $Q_2(\omega)$:
\begin{equation}
    Z_{\mathrm{in}}(\omega) = \frac{Q_1(\omega)}{Q_2(\omega)}
\end{equation}
Note that this holds regardless of where the voltage and the currents are defined.  

Therefore, the scattering coefficient $S_{11}(\omega)$ has the following structure:
\begin{equation}
    S_{11}(\omega) := \frac{Z_{\mathrm{in}}(\omega)/Z_0 - 1}{Z_{\mathrm{in}}(\omega)/Z_0 + 1} = \frac{Q_1(\omega) - Z_0 Q_2(\omega)}{Q_1(\omega)+Z_0 Q_2(\omega)}
\end{equation}
Therefore, it can be represented as a ratio of two complex polynomials of the \emph{same} degrees (called it $K$). 

By the fundamental theorem of algebra, the log spectrum of $S_{11}(\omega)$ is:
\begin{equation}
\label{eq:czp_freq_response_appendix}
\log|S_{11}(\omega)| = \log|c_0(\vh)| + \sum_{k=1}^K \log\frac{|\omega - z_k(\vh)|}{|\omega - p_k(\vh)|}
\end{equation}
where the constant $c_0(\vh)$, zeros $\{z_k(\vh)\}_{k=1}^K$ and poles $\{p_k(\vh)\}_{k=1}^K$ are all functions of $A=A(\vh)$ and thus the design choice $\vh$.
\end{proof}
\textbf{Remark}. Note that it is possible that the polynomial $Q_1(\omega) - Z_0 Q_2(\omega)$ may not have the same order as the polynomial $Q_1(\omega) + Z_0 Q_2(\omega)$ (i.e., one of them has their leading term precisely cancelled out, while the other does not). While this is a rare situation, when it happens, Eqn.~\ref{eq:czp_freq_response_appendix} still applies, by having one or more zeros/poles moving far away from the concerned frequency region of $\omega$. Then the corresponding factor $|\omega-z_k|$ (or $|\omega-p_k|$) almost never changes, and can be absorbed into the constant term $c_0$. 

\section{Specification of the Design Space}\label{antenna_specs}
The dimensions $s_k = (s_{k,x}, s_{k,y})$ and ranges for the location $l_{k,x}$, $l_{k,y}$ of each of the 5 patches $p_k$ are \\

$p_1$: $s_{1}=(0.75, 5.49)$, $l_{1,x} \in [0,10]$, $l_{1,y} \in [0.5,0.5]$\\
$p_2$: $s_{2}=(17.64,1.7)$, $l_{2,x} \in [0,12.36]$, $l_{2,y} \in [1,4.7]$\\
$p_3$: $s_{3}=(11.38,3.0)$, $l_{3,x} \in [10,18.62]$, $l_{3,y} \in [1,3]$\\
$p_4$: $s_{4}=(18.63,0.56)$, $l_{4,x} \in [0,11.37]$, $l_{4,y} \in [1,5.44]$\\
$p_5$: $s_{5}=(0.99,2.43)$, $l_{5,x} \in [10,29.01]$, $l_{5,y} \in [-2,3.57]$

where all values are in mm. These values were determined from antennas that have been used for past production devices in industry. Additionally, patch $p_1$ determines the location of the discrete port.

\begin{algorithm}
\caption{\textsc{Add\_Patch\_To\_Image($x_{bl}$, $y_{bl}$, $x_{tr}$, $y_{tr}$, $image$)}}
\label{imgen}
	{\bf Input:} $x_{bl}$, $y_{bl}$, $x_{tr}$, $y_{tr}$: floating point bottom-left and top-right coordinates of rectangular patch, $image$: Image array
	\begin{algorithmic}[1]
	    \STATE // Lower bound on coordinates to index image
	    \STATE $\bar{x}_{bl} = \lfloor x_{bl}\rfloor$, $\bar{y}_{bl} = \lfloor y_{bl}\rfloor$, $\bar{x}_{tr} = \lfloor x_{tr}\rfloor$, $\bar{y}_{tr} = \lfloor y_{tr}\rfloor$
	    \STATE // Write X boundaries
	    \STATE $image_{x\_bound}[\bar{y}_{bl}:\bar{y}_{tr}][\bar{x}_{bl}] = 1 - (x_{bl} - \bar{x}_{bl})$
	    \STATE $image_{x\_bound}[\bar{y}_{bl}:\bar{y}_{tr}][\bar{x}_{tr}] = x_{tr} - \bar{x}_{tr}$
	    \STATE // Write Y boundaries
	    \STATE $image_{y\_bound}[\bar{y}_{bl}][\bar{x}_{bl}:\bar{x}_{tr}] = 1 - (y_{bl} - \bar{y}_{bl})$
	    \STATE $image_{y\_bound}[\bar{y}_{tr}][\bar{x}_{tr}:\bar{x}_{tr}] = y_{tr} - \bar{y}_{tr}$
	    \STATE // Write interior
	    \STATE $image_{interior}[\bar{y}_{bl} + 1 : \bar{y}_{tr} - 1][\bar{x}_{tr} + 1 :\bar{x}_{tr} - 1] = 1.0$
	\end{algorithmic}
\end{algorithm}

\section{Experimental details}\label{hyperparams}
\begin{table}[h]
\begin{center}
  \begin{tabular}{ c | c }
    Hyperparameter & Value \\ 
    \hline
    Batch Size & 100\\
    Learning Rate & .0005 \\
    Activation & Swish\\
    Warmup epochs & 100 \\
    Decay LR plateau epochs & 20 \\
    Decay LR plateau factor & .5 \\
    Total Epochs & 500 \\
    Attention heads & 8\\
    Attention layers & [2,4,6,8]\\
    \czp{} Degree & [8,12,16,20]
  \end{tabular}
\end{center}
\caption{General experimental hyperparameters}
\end{table}
\begin{table}[h]
\begin{center}
  \begin{tabular}{ c | c }
    Hyperparameter & Value \\ 
    \hline
    Spatial Attention Maps & 16\\
    Conv KernelxStridexPad & 5x1x2\\
    Conv Layers & 2\\
    Conv Filters & 128\\
    Attn dim\_feedforward & 256
  \end{tabular}
\end{center}
\caption{Image input specific hyperparameters for the transformer with spatial attention}
\end{table}

\begin{table}[h]
\begin{center}
  \begin{tabular}{ c | c }
    Hyperparameter & Value \\ 
    \hline
    FC embed dimension & 256\\
    FC layers & 2\\
    Attn dim\_feedforward & 512
  \end{tabular}
\end{center}
\caption{Coordinate input specific hyperparameters for the transformer with coordinate input}
\end{table}


\section{Other Visualizations}\label{qualitative}
\begin{figure*}
\begin{tabular}{c c}
\includegraphics[width=.48\textwidth]{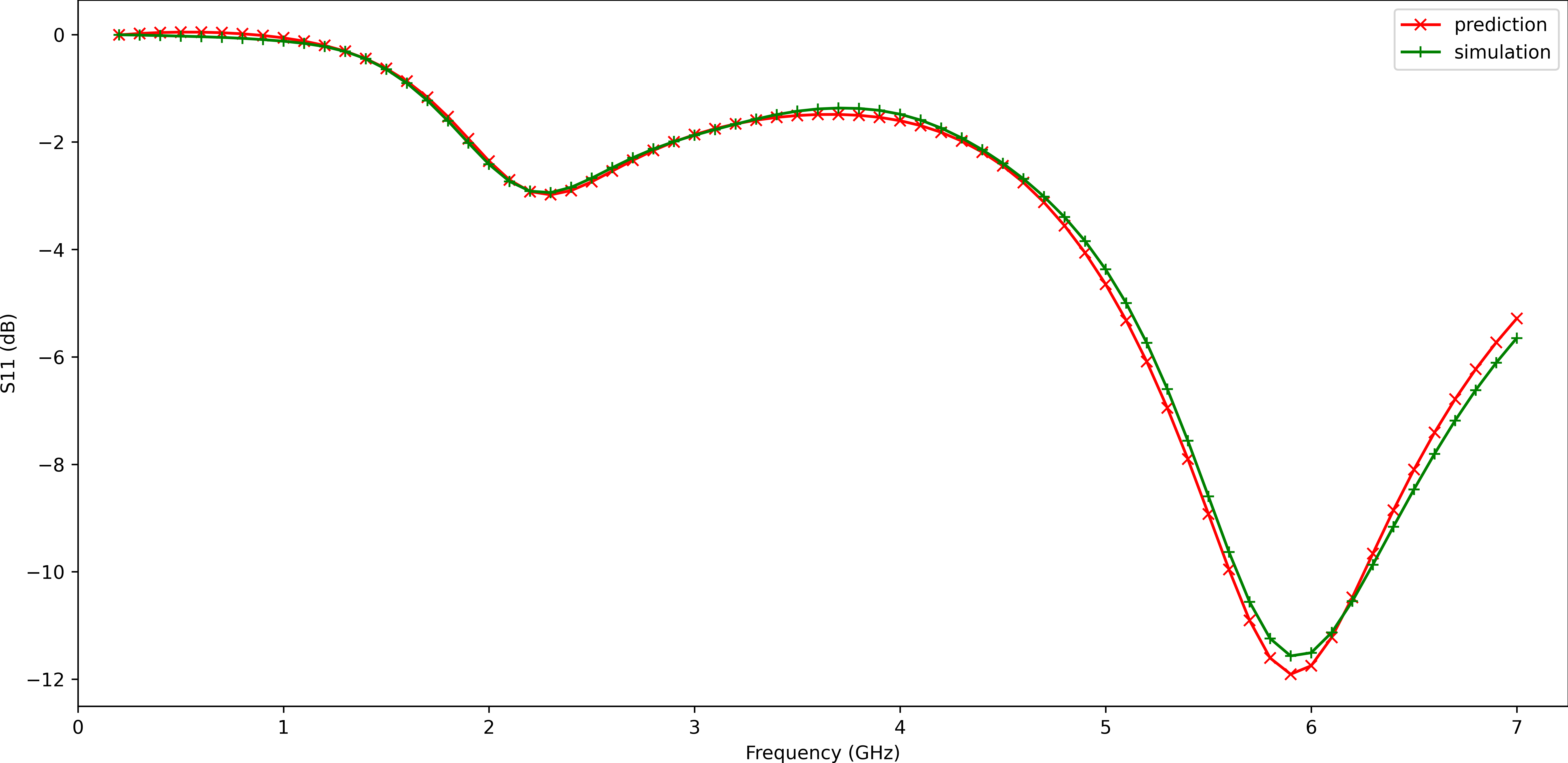}&\includegraphics[width=.48\textwidth]{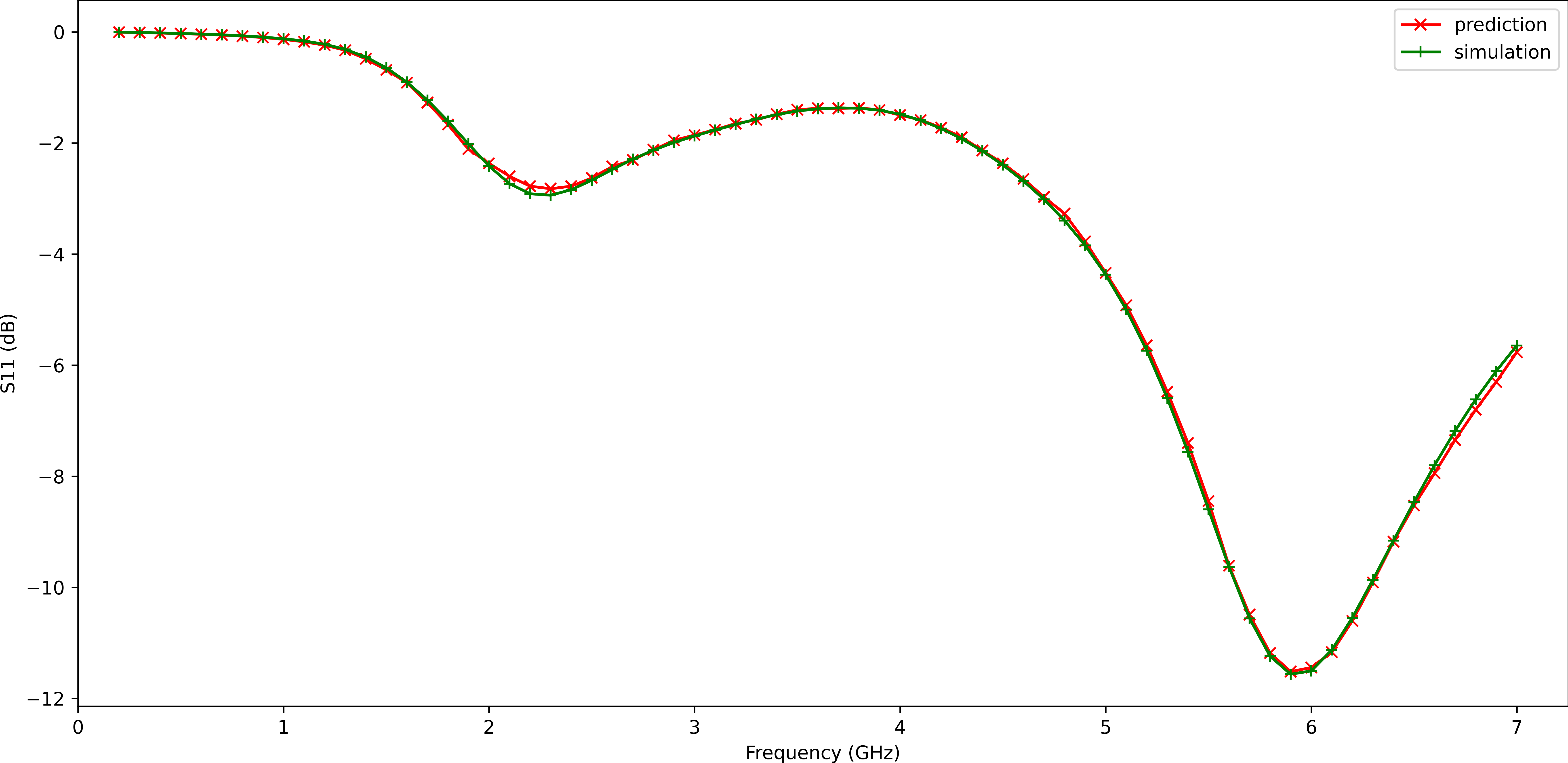}\\
\includegraphics[width=.48\textwidth]{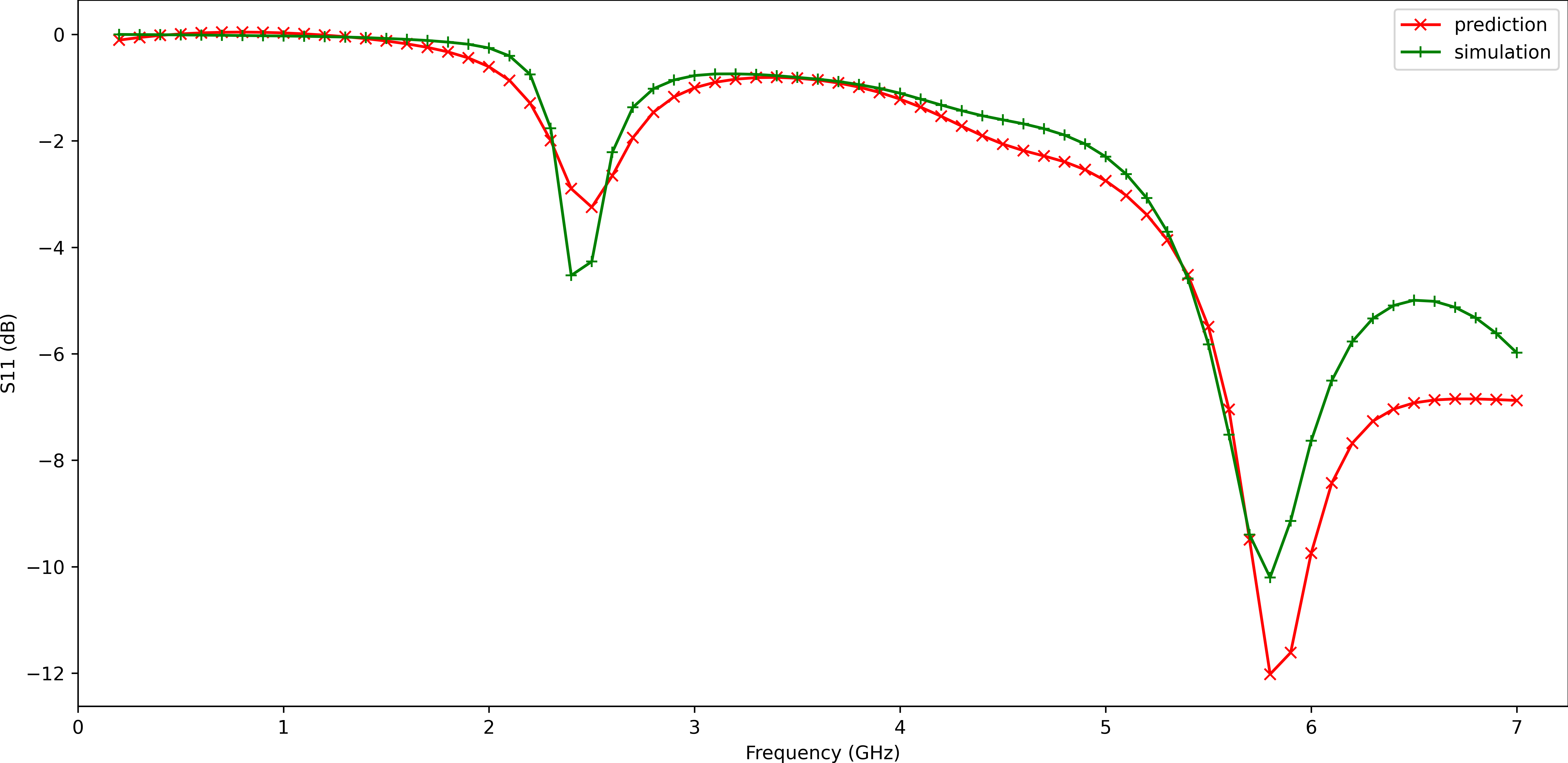}&\includegraphics[width=.48\textwidth]{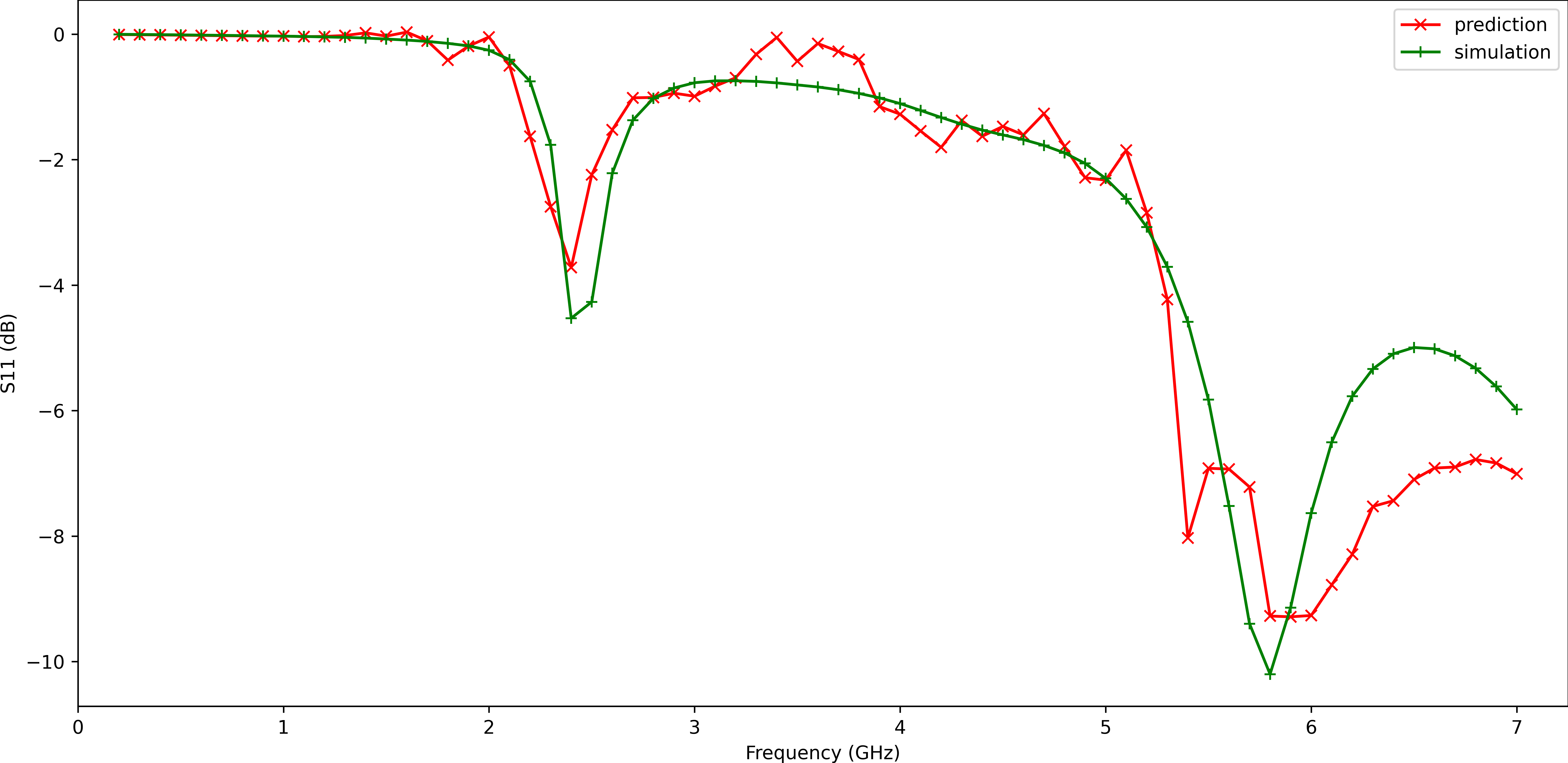}\\
{\bf (a)} \czp & {\bf (b)} Raw\\
\end{tabular}
\caption{Comparison of predicted frequency response {\bf (a)} \czp{} and {\bf (b)} Raw with an 8 layer transformer and image input on two {\it test set} examples, with low MSE (top row) and high MSE (bottom row). Red is the frequency response predicted by the surrogate model, green is the frequency response from CST. For easy examples, \czp{} and raw are both smooth. For hard examples, \czp{} is smooth by design but raw may be non-smooth.}\label{smooth_vs_nonsmooth}
\end{figure*}

\begin{figure}
\begin{tabular}{c}
\includegraphics[width=.48\textwidth]{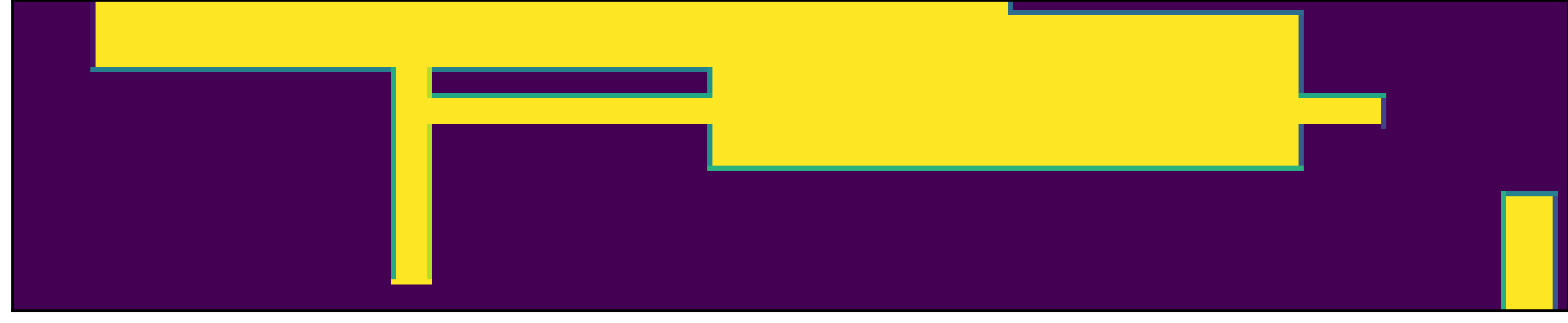}\\
\hline \\
\includegraphics[width=.48\textwidth]{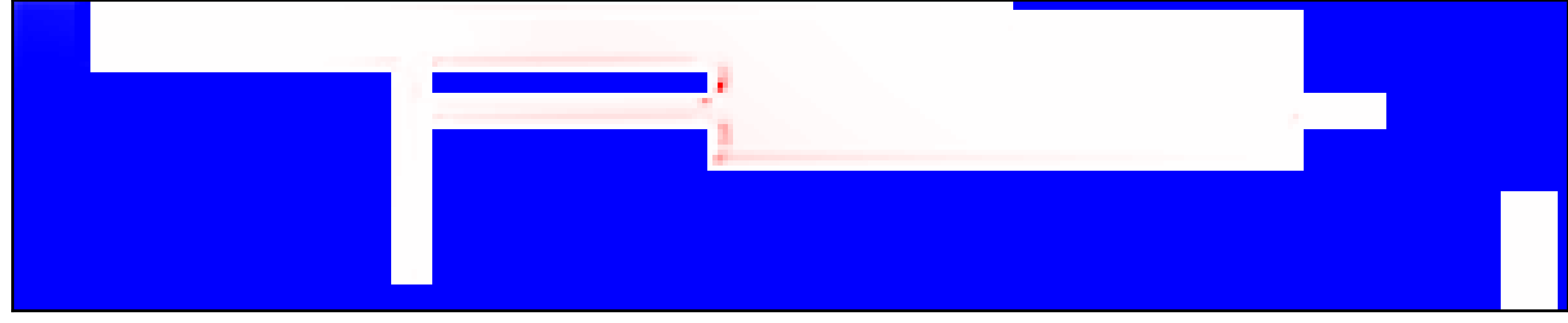}\\
\includegraphics[width=.48\textwidth]{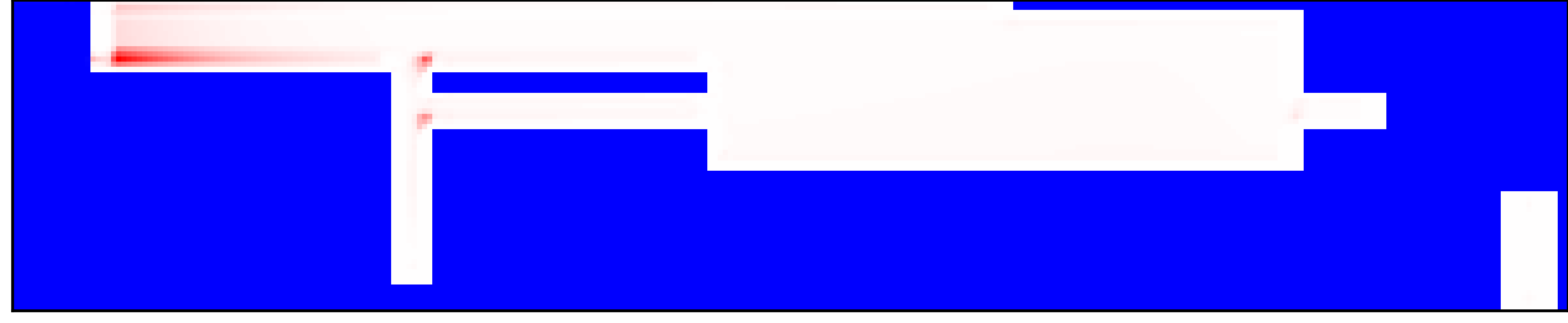}\\
\includegraphics[width=.48\textwidth]{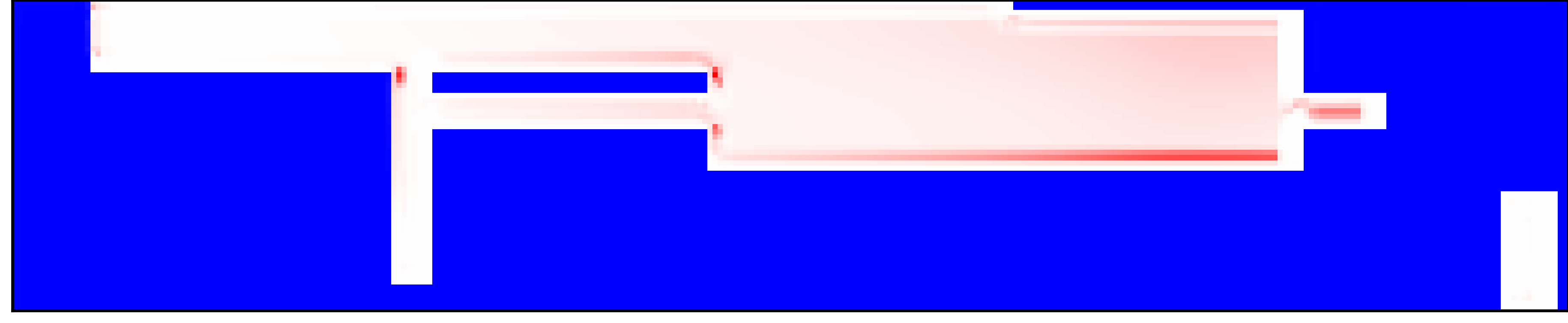}\\
\includegraphics[width=.48\textwidth]{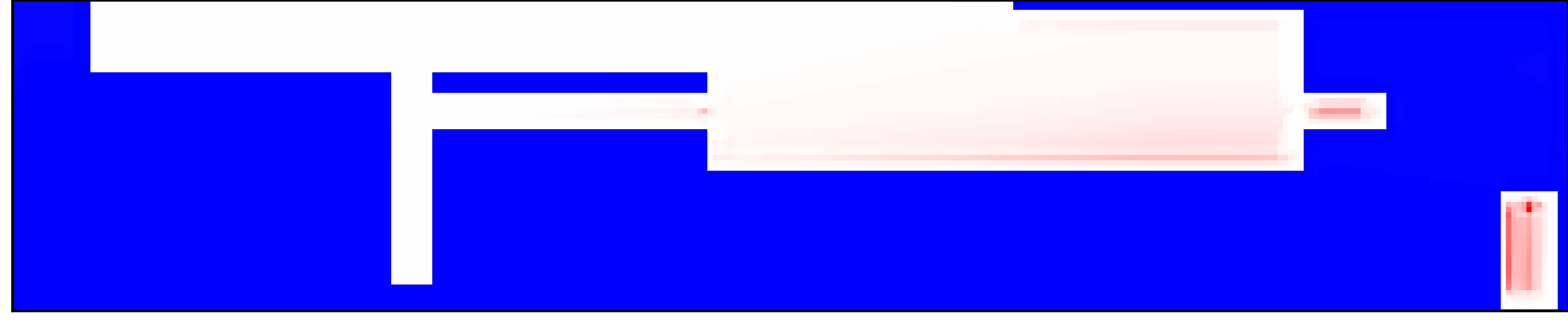}\\
\end{tabular}
\caption{A random antenna configuration (above line) and four (of sixteen total) attention maps (below line) learned by the spatial attention component of the proposed transformer architecture (overlayed on the antenna configuration). The intensity of the red shading indicates the activations of the attention map. Activations are greatest around boundaries and corners.}\label{attn_maps}
\end{figure}
\end{document}